%% file: Network_Recovery.tex
\documentclass{acm_proc_article-sp}
\usepackage{times}
\usepackage{graphicx}
\usepackage{subfigure} 
\usepackage{algorithm}
\usepackage{algorithmic}
\usepackage{natbib}
\usepackage{hyperref}

\usepackage{booktabs}
\usepackage{url}
\usepackage{color}
\usepackage{psfrag}
\usepackage{amsmath,amsfonts,amssymb}

\input{MACPdef}
\input{MACPdef-ams}

\newcommand{\bbI}{\ensuremath{\mathbb{I}}}
\newcommand{\ind}[1]{\ensuremath{\bbI\left(#1\right)}}
\newcommand{\bOmega}{\mathbf{\Omega}}
\newcommand{\Xij}{\X_{ij}}
\newcommand{\tXij}{\tilde{\X}_{ij}}
\newcommand{\tX}{\tilde{\X}}
\newcommand{\Eij}{\E_{ij}}
\newcommand{\Oij}{\bOmega_{ij}}
\newcommand{\condP}[2]{\mathbb{P}\left( #1 \middle\vert #2 \right)}
\newcommand{\ExpO}[1]{\bbE_{\bOmega}\left[ #1 \right]}
\newcommand{\ExpX}[1]{\bbE_{\tX}\left[ #1 \right]}

\newtheorem{theorem}{Theorem}

\newtheorem{lemma}{Lemma}

\begin{document}

\title{Network Inference by Learned Node-Specific Degree Prior}
%
%
%
%
%

\numberofauthors{4} 
%
\author{
%
%
\alignauthor
Qingming Tang \footnotemark[1]\\
       \affaddr{Toyota Technological Institute at Chicago}\\
       \affaddr{6045 S. Kenwood Ave.}\\
       \affaddr{Chicago, Illinois 60637}\\
       \email{qmtang@ttic.edu}
\alignauthor
Lifu Tu \footnotemark[1]\\
       \affaddr{Toyota Technological Institute at Chicago}\\
       \affaddr{6045 S. Kenwood Ave.}\\
       \affaddr{Chicago, IL}\\
       \email{lifu@ttic.edu}
\and  
\alignauthor
Weiran Wang \footnotemark[1] \\
       \affaddr{Toyota Technological Institute at Chicago}\\
       \affaddr{6045 S. Kenwood Ave.}\\
       \affaddr{Chicago, IL}\\
       \email{weiranwang@ttic.edu }
\alignauthor Jinbo Xu \\
       \affaddr{Toyota Technological Institute at Chicago}\\
       \affaddr{6045 S. Kenwood Ave.}\\
       \affaddr{Chicago, Illinois 60637}\\
       \email{jinbo.xu@gmail.com}
}


\maketitle

\begin{abstract}
We propose a novel method for network inference from partially observed edges using a node-specific degree prior. The degree prior is derived from observed edges in the network to be inferred, and its hyper-parameters are determined by cross validation.
Then we formulate network inference as a matrix completion problem regularized by our degree prior. Our theoretical analysis indicates that this prior favors a network following the learned degree distribution, and may lead to improved network recovery error bound than previous work.
Experimental results on both simulated and real biological networks demonstrate the superior performance of our method in various settings.
\end{abstract}

\footnotetext[1]{Equal Contribution.}

\section{Introduction}

Network inference or structure learning has been widely studied in machine learning. There are two typical scenarios. In the first scenario, the task is to estimate the structure of an undirected graphical model from a high-dimensional dataset, e.g., learning gene co-expression networks from gene expression data \cite{marbach2012wisdom,de2010advantages,zhang2005general}. This task has been extensively studied, with a popular method being the graphical Lasso which assumes the underlying graph to be sparse \cite{meinshausen2006high, yuan2007model, friedman2008sparse,banerjee2008model, wainwright2006high}.
If prior information regarding clusters or blocks of the network is available, one may apply group penalties to promote desired patterns among the edges in a cluster \cite{friedman2010note, mohan2014node, tibshirani2005sparsity}. More recently, structure inducing norms/priors \cite{bach2010structured,candes2008enhancing} that promote a hub structure \cite{mohan2014node,tan2014learning,mohan2012structured} or a scale-free network \cite{liu2011learning,tang2015learning,defazio2012convex} are proposed. \par

In the second scenario of network inference, the task is to reconstruct the whole network or predict missing links based on a subset of observed edges, e.g., link prediction in a social network \cite{liben2007link} and inferring unknown protein-protein interactions (PPIs) from experimentally-validated PPIs \cite{wang2013predicting,dai2010low}. There are many solutions to this problem given different assumptions on data generation. For example, the problem can be solved by using influence cascade models under the diffusion process assumption \cite{daneshmand2014estimating, pouget2015inferring}. Another popular approach is to formulate the problem as a matrix completion problem \cite{wang2013predicting,huang2013robust,hsieh2014pu}, in which the matrix to be completed is often assumed to have certain structural properties, e.g., being low-rank \cite{candes2009exact,ding2006orthogonal,mnih2007probabilistic,cai2010singular}.
This is reasonable since many real-world networks have only a small number of degree of freedom.

This paper focuses on the second scenario, i.e., predicting the whole network from a subset of observed edges. We formulate such a network inference problem as a matrix completion problem regularized by our novel node-specific degree prior. We learn the  degree of an individual node from the observed edges by cross-validation, so that the learned degree is (approximately) consistent with the partial observation. Considering that the observed degree distribution may be different from the true degree, we use a soft rather than hard degree constraint in our prior.

To justify our method, we show theoretically that our node-specific degree prior indeed can help induce a network that follows a given degree distribution. Under reasonable assumptions on the observation process, we provide upper bound on the expected recovery error of our network inference algorithm, which is superior than the error bound of existing approach due to the additional regularization by our degree prior. Furthermore, we experiment with real biological networks data, and show that by tuning the hyper-parameters of our degree-based prior from the observed edges, the proposed method can obtain much better performance than existing methods for network inference. \par

The rest of the paper is organized as follows. In section \ref{sec:prior}, we introduce our node-specific degree prior, and show that it can induce a graph following a given degree distribution. 
In section \ref{sec:optimization}, we formulate the network inference problem as a regularized matrix completion problem, and present its optimization procedure. 
In section \ref{sec:related}, we discuss the difference between our method and some closely related works. In section \ref{sec:results}, we compare our method with related methods in terms of prediction accuracy on both synthetic and real datasets. Finally, we provide concluding remarks in section \ref{sec:conclusion}.

\paragraph{Notations.}
Let $G=(V,E)$ denote the underlying network to be inferred, where $V$ is the set of $p$ vertices ($|V|=p$) and $E$ the true edge set. We use $\E\in \bbR^{p\times p}$ to denote the adjacency matrix of $G$, i.e., $\E_{ij} = 1$ if and only if $(i,j)\in {E}$. In this paper, we assume $\mathbf{E}_{ii} = 0$ for all $i$, i.e., there is no edge connecting each node to itself. 
We denote by ${\Omega}\subseteq E$ the set of observed edges, and $\mathbf{\Omega}\in \bbR^{p\times p}$ the indicator matrix such that $\bOmega_{ij} = 1$ if and only if $(i,j)\in {\Omega}$.

To infer the underlying graph $G$ from ${\Omega}$, we first estimate a real-valued symmetric matrix $\X$ from $\mathbf{\Omega}$, each entry of which indicates the strength of the existence of an edge between two vertices. We then predict the largest $K$ entries of (the upper-diagonal of) $\X$ as edges, where $K$ is the number of desired edges for the whole network. Alternatively, we can use a thresholding procedure to predict edges from $\X$.

\section{Network Inference by Node-Specific Degree Prior}\label{sec:prior}

In this section, we propose a new degree-based prior to regularize the network inference problem. We will show that this degree prior helps to induce a network following the desired node-specific degree distribution. We also estimate the recovery error bound of the constrained network inference problem under reasonable assumptions.

\subsection{Inducing the desired degree distribution}
There have been several degree priors/norms for inducing a scale-free (or hub) network \cite{liu2011learning,defazio2012convex,mohan2012structured,tang2015learning}. But most of these priors work at the global level, e.g., they assume degrees of all nodes approximately follows a power-law distribution.
In constrast to these existing work, we would like to use the (noisy) degree information at each individual node. As mentioned before, we assume that the target network structure is implied by a real-valued symmetric matrix $\X \in \mathbb{R}^{p\times p}$, and we may use the top $K$ entries in $\X$ as the predicted edges (including observed edges). 
Let $\mathbf{d}=(d_1,d_2,...,d_p)$ be the desired degree of the $p$ nodes. We propose the following prior for $\X$:
\begin{align}
  & S_H(\X,\mathbf{d},\alpha) \nonumber \\
  & =\sum_{i=1}^p{\frac{H_1^{\alpha}\X_{i,[1]}+H_2^{\alpha}\X_{i,[2]}+...+H_{p-1}^{\alpha}\X_{i,[p-1]}}{H_{d_i}^{\alpha}}}
  \label{degree_inducing}
\end{align}
where $\alpha>0$ is a hyper-parameter and $H$ is a monotonically increasing positive sequence, i.e., $0<H_1<H_2<...<H_{p-1}$. In this paper, we use $H_i=\log(i+1)$ for $1\leq i < p$, although there are many choices for $H$ as long as it increases at a moderate rate. $\X_{i,[j]}$ denotes the $j^{th}$ largest entry in the $i^{th}$ row of $\X$ excluding $\X_{ii}$. That is, the likelihood of $[j]$ and $i$ forming an edge is ranked in the $j^{th}$ position among all the possible edges adjacent to $i$.

Theorem \ref{property} below shows that the prior in \eqref{degree_inducing} favors a graph that follows a given degree distribution. 
\begin{theorem} \label{property}
Let $\mathbf{d}=(d_1,d_2,...,d_p)$ be the degree of the underlying network where $\sum_{i=1}^p{d_i}=2K$. Let $\X$ be a real-valued symmetric matrix with non-negative entries and $\mathbf{d}'=(d'_1,d'_2,...,d'_p)$ ($\sum_{i=1}^p{d'_i}=2K$) be the degree distribution of the network resulting from  the top $K$ upper-diagonal entries of $\X$. If $\mathbf{d}'\neq \mathbf{d}$, then there exists another symmetric matrix $\X^* \in \mathbb{R}^{p\times p}$ satisfying the following conditions: (1) $\X^*$ has the same set of entries as $\X$; (2) the network derived from the top $K$ entries of $\X^*$ has degree $\mathbf{d}=(d_1,d_2,...,d_p)$; and (3) $S_H(\X^*,\mathbf{d},\alpha)\leq S_H(\X,\mathbf{d},\alpha)$.
\end{theorem}

\begin{proof}
Let $S_H^i (\mathbf{X},\mathbf{d},\alpha)$ be the sum of the terms associated with node $i$ in \eqref{degree_inducing}. Then we have
\begin{align*}
  S_H^i (& \mathbf{X},\mathbf{d},\alpha) \\
  & =\frac{H_1^{\alpha}}{H_{d_i}^{\alpha}} \X_{i,[1]} + \frac{H_2^{\alpha}}{H_{d_i}^{\alpha}} \X_{i,[2]}+...+ \frac{H_{p-1}^{\alpha}}{H_{d_i}^{\alpha}} \X_{i,[p-1]},
\end{align*}
which is a linear combination of the (sorted) entries of the $i^{th}$ row of $\mathbf{X}$. By the definition of  $H$, we have
\begin{align}\label{e:H-increasing}
\frac{H_1^{\alpha}}{H_{d_i}^{\alpha}} < \frac{H_2^{\alpha}}{H_{d_i}^{\alpha}} < \dots
< \frac{H_{d_i}^{\alpha}}{H_{d_i}^{\alpha}}=1 < \dots < \frac{H_{p-1}^{\alpha}}{H_{d_i}^{\alpha}}.
\end{align}

We now construct $\mathbf{X}^*$ from $\mathbf{X}$ as follows. First, we sort all the upper-diagonal elements of $\mathbf{X}$ descendingly to obtain the below sequence $\Y$:
\begin{align*}
Y_1 \ge Y_2 \ge \dots \ge Y_{K} \ge \dots \ge Y_{\frac{p(p-1)}{2}} \ge 0.
\end{align*}
Then we employ Algorithm~\ref{alg} to place each element of $\Y$ in a descending order into two entries $(i,j)$ and $(j,i)$ in $\mathbf{X}^*$. $Y_s$ appears twice in $S_H(\X^*,\mathbf{d},\alpha)$: one in $S_H^i (\X^*,\mathbf{d},\alpha)$ with coefficient $\frac{H_{[s_i]}^{\alpha}}{H_{d_i}^{\alpha}}$ where $[s_i]$ is the ranking of this entry in row $i$, and the other in $S_H^j (\X^*,\mathbf{d},\alpha)$ with coefficient $\frac{H_{[s_j]}^{\alpha}}{H_{d_j}^{\alpha}}$. Thus the contribution of $Y_s$ in $S_H(\X^*,\mathbf{d},\alpha)$ is $\left( \frac{H_{[s_i]}^{\alpha}}{H_{d_i}^{\alpha}} + \frac{H_{[s_j]}^{\alpha}}{H_{d_j}^{\alpha}} \right) Y_s$.

\begin{algorithm}
\caption{\textbf{Construction of $\mathbf{X}^*$}.}
  \label{alg}
  \renewcommand{\algorithmicrequire}{\textbf{Input:}}
  \renewcommand{\algorithmicensure}{\textbf{Output:}}
  \begin{algorithmic}[1]
    \REQUIRE: $H$, $\alpha$, degree distribution $\mathbf{d}$, and $\mathbf{X}$.
    \STATE Initialize an array $A$ of $p$ elements, $A[i]\leftarrow 0,\quad \forall 1\le i\le p$.
    \STATE Sort the upper-diagonal entries of $\mathbf{X}$ into sequence $\Y$ in a descending order.
    \FOR{$s=1,2,\dots,\frac{p(p-1)}{2}$}
    \STATE Find two indices $i$ and $j$ such that $\frac{H_{A[i]+1}^{\alpha}}{H_{d_i}^{\alpha}} + \frac{H_{A[j]+1}^{\alpha}}{H_{d_j}^{\alpha}}$ is the smallest.
    \STATE Set $\mathbf{X}^*_{ij}\leftarrow Y_s, \quad \mathbf{X}^*_{ji}\leftarrow Y_s$.
    \STATE Update $A[i]\leftarrow A[i]+1,\quad A[j]\leftarrow A[j]+1$.
    \ENDFOR
    \ENSURE $\mathbf{X}^*$.
  \end{algorithmic}
\end{algorithm}

We now prove that $\X^*$ has the desired properties.
\paragraph{$\mathbf{X}^*$ is symmetric and has the same set of entries as $\mathbf{X}$.}
It is clear from step~5 of the algorithm that the resultant $\mathbf{X}^*$ is symmetric. Since in each iteration we assign a different entry of $\mathbf{X}$ to that of $\mathbf{X}^*$, the two matrices have the same set of entries.

\paragraph{$\mathbf{X}^*$ has the desired degree distribution $\mathbf{d}$.}
The network resulting from $\X^*$ is determined from the first $K$ iterations of Algorithm~\ref{alg}.
\eqref{e:H-increasing} indicates that $\frac{H_{p}^{\alpha}}{H_{d_i}^{\alpha}} \le 1 \le \frac{H_{q}^{\alpha}}{H_{d_i}^{\alpha}}$ for any $p\le d_i \le q$, so the following set $\mathcal{K}$ contains the smallest $\sum_{i=1}^p d_i=2K$ coefficients in $S_H(\mathbf{X},\mathbf{d},\alpha)$.
\begin{align*}
\mathcal{K} = \left\{ \frac{H_{1}^{\alpha}}{H_{d_1}^{\alpha}},\dots, \frac{H_{d_1}^{\alpha}}{H_{d_1}^{\alpha}}, \frac{H_{1}^{\alpha}}{H_{d_2}^{\alpha}},\dots,\frac{H_{d_2}^{\alpha}}{H_{d_2}^{\alpha}},\dots, \frac{H_{1}^{\alpha}}{H_{d_p}^{\alpha}},\dots,\frac{H_{d_p}^{\alpha}}{H_{d_p}^{\alpha}} \right\}
\end{align*}
Since step~4 of Algorithm~\ref{alg}  selects a pair of indices $i$ and $j$ with the smallest $\frac{H_{A[i]}^{\alpha}}{H_{d_i}^{\alpha}}$ and $\frac{H_{A[j]}^{\alpha}}{H_{d_j}^{\alpha}}$, these selected coefficients must be chosen from $\mathcal{K}$. As a result, upon the termination of the $K^{th}$ iteration in Algorithm~\ref{alg}, we have $A[i]=d_i$, $\forall 1\le i\le p$, and $\X^*$ has the desired degree distribution.

\paragraph{$S_H(\mathbf{X}^{*},\mathbf{d},\alpha)\le S_H(\mathbf{X},\mathbf{d},\alpha)$, i.e., $\X^*$ has a smaller penalty than $\X$.} Note that both $\mathbf{X}$ and $\mathbf{X}^{*}$ have the same set of ranked entries $Y_1 \ge Y_2,\dots,\ge Y_{\frac{p(p-1)}{2}}\ge 0$.
A larger entry in $\mathbf{X}^{*}$ always has a smaller coefficient in $\mathbf{X}^{*}$, and all entries and coefficients are non-negative, so the resultant $\mathbf{X}^{*}$ has the smallest penalty among all the matrices with the same set of entries.
\end{proof}

Theorem \ref{property} shows that \eqref{degree_inducing} is a structure-inducing prior that favors a network following a given distribution $\mathbf{d}$. We can tune the parameter $\alpha$ to control the impact of the prior. The induced graph would almost follow the degree distribution when $\alpha$ is large (e.g. $\alpha \to \infty$), as the penalty coefficient $\frac{H_j^{\alpha}}{H_{d_i}^{\alpha}}\rightarrow \infty$ for $\X_{i,[j]}$ where $j > d_i$. On the other hand, when $\alpha$ is small, the prior only weakly encourages the given degree distribution. Furthermore, when $\alpha \to 0$, the prior reduces to the $l_1$ norm. \par
Although \eqref{degree_inducing} induces a graph following a given degree distribution, it is a soft constraint and tolerates some noise in the degree distribution. In fact, the difference in penalty coefficients of two entries $\X_{i,[t]}$ and $\X_{i,[t+\delta]}$ is bounded as
\begin{align}
  \frac{H^{\alpha}_{t+\delta}-H^{\alpha}_{t}}{H^{\alpha}_{d_i}}\leq \frac{\log^{\alpha}(\delta+1)-\log^{\alpha}(1)}{\log^{\alpha}(d_i+1)}=\left( \frac{\log(\delta+1)}{\log(d_i+1)} \right)^{\alpha}
  \label{bound}
\end{align}
which is very small when $d_i$ is large and $\delta$ is small.
This bound implies the following properties of our prior in \eqref{degree_inducing}.
\begin{enumerate}
\item The larger the degree of a node is, i.e., the more neighbors it has, the smoother its penalty coefficients are.
\item When two entries in the same row have similar values and thus similar rankings, their corresponding penalty coefficients in the regularizer are also close. 
\end{enumerate}
These smoothness properties are further fine-tuned via the parameter $\alpha$.

\subsection{Recovery Error Bound}

Mathematically, the network inference problem addressed in this paper is related to PU (Positive-Unlabeled) learning for matrix completion \cite{hsieh2014pu}. \par
Here instead of picking top $K$ entries, we assume that we derive the network from a real-valued symmetric matrix $\X^*$ by thresholding (the two ways are closely related). We assume that the 0/1 adjacency matrix $\E$ is observed from $\X^*$ by a thresholding process: Let $q\in \bbR$ be the threshold value, then $\E_{ij}= thr(\X^*_{ij}):=\ind{ \X^*_{ij} > q }$ where $\ind{\cdot}$ is the indicator function. Furthermore, we assume that a subset of the edge set $E=\{(i,j): \E_{ij}=1\}$ is observed by uniformly sampling elements from $E$ with probability $\rho\in (0,1)$. Recall that we use the matrix $\bOmega\in \bbR^{p\times p}$ to denote the observations, where $\Oij=1$ if the edge $(i,j)$ is observed and $\Oij=0$ otherwise.

We assume the underlying real-valued matrix $\X^*$ comes from the set $\calX$, defined as
\begin{align}
\calX:=\{& \X \in \bbR^{p\times p} \;\mid\; \X=\X^\top, \nonumber \\
 \quad 0\le \X \le 1, \quad & \norm{\X}_*\le t, \quad S_{H}(\X,\mathbf{d},\alpha)\le r \}, \nonumber
\end{align}
Where $0\le \X \le 1$ is elementwise comparison and $\norm{\X}_*$ is the nuclear norm (sum of singular values).
We assume the underlying matrix to have small nuclear norm as a proxy to being low-rank, which is a common approach for matrix completion \cite{srebro2005rank}. We use the following loss function to estimate $\X^*$ from $\calX$: 

\begin{align} \label{e:obj}
 L(\X;\bOmega) \nonumber & := \left( 1 - \frac{\rho}{2} \right) \sum_{(i,j):\Oij=1}\; \left(\Xij-1 \right)^2 \nonumber \\
 & + \frac{\rho}{2} \sum_{(i,j):\Oij=0} \left(\Xij-0\right)^2
\end{align}
This loss function assigns two different weights to the squared error between $\Xij$ and $\Oij$ depending on the observed value $\Oij$. When the percentage of unobserved edges is high (i.e., $\rho$ is small), the loss function weighs more on correctly predicting the observed $1$'s and allows for larger errors for unobserved entries; otherwise it tends to predict smaller values for the unobserved entries. \par
Let $\hat{\X}$ minimize the loss function subject to the constraint that $\X \in \calX$. We may predict an edge set from $\hat{\X}$ by thresholding. That is, there is one edge between $i$ and $j$ if and only if $\ind{\hat{\X}_{ij}>q}=1$. We can estimate the \emph{expected recovery error} of $\hat{\X}$ as follows.
\begin{align} \label{expected-recovery-error}
R(\hat{\X})=\ExpO{ \sum_{i,j} \ind{ thr(\hat{\X}_{ij})\neq \Eij } }
\end{align}
where the expectation is taken over random selection of the observed edge set $\Omega$. We have the following theorem regarding the bound of the {expected recovery error}.

\begin{theorem}\label{mistake-bound}
Assume the underlying graph has a degree distribution $\mathbf{d}^* =(d_1^*,\dots,d_p^*)$, with $s=2 \abs{E}=\sum_{i=1}^p d_p^*$ and $d_{\max}^*=\max\ (d_1^*,\dots,d_p^*)$. Let $q < \frac{2-\rho}{3-2\rho}$ be the threshold value used to obtain edges from the underlying matrix $\X^*$. Let $\hat{\X}(\bOmega)$ be the minimizer of the loss function defined in \eqref{e:obj} subject to the constraint that $\X \in \calX$. Assume that the regularizer $S_H(\X,\mathbf{d},\alpha)$ uses an estimated degree distribution $\mathbf{d}=(d_1,\dots,d_p)$ with $d_{\max}=\max\ (d_1,\dots,d_p)$. Then with probability at least $1-\delta$, we have
\begin{align*}
R(\hat{\X}) \le \frac{4\gamma (2-\rho)}{\rho}  \left(
2 \min\left( A,B \right) + \sqrt{\frac{s \log 2/\delta}{2}}
\right)
\end{align*}
where $A=t C (2 \sqrt{d_{\max}^*} + \sqrt[4]{s})$, $B= r \log_2^\alpha (d_{\max}+1) $, $\gamma=\max\left(\frac{1}{q^2}, \frac{1}{(3-2\rho)\left(q - \frac{2-\rho}{3-2\rho} \right)^2} \right)$ and $C$ is a universal constant.
\end{theorem}

This theorem shows that on average, the number of mistakes (i.e., $R(\hat{\X})$) made in the recovered network is bounded and the bound does not increase with respect to the matrix size $p^2$, but rather (mildly) depends on the maximum degree and the number of edges, and the complexity of the matrix. In case of a sparse network, the average error $\frac{R(\hat{\X})}{s}$ has a very small bound. See our supplementary material for a detailed proof of Theorem~\ref{mistake-bound}. Part of our proof follows the techniques used by \citet{hsieh2014pu}. That is, we relate the {expected recovery error} \eqref{expected-recovery-error} to the following \emph{label-dependent recovery error}.
\begin{align*}
R_{\rho} (\X;\bOmega) & = \sum_{i,j} \left\{ \left( 1-\frac{\rho}{2}\right) \ind{thr(\Xij)=0} \ind{\Oij=1} \right. \\
& + \left. \frac{\rho}{2} \ind{thr(\Xij)=1} \ind{\Oij=0} \right\},
\end{align*}
which is further related to the loss function defined in \eqref{e:obj}.
Then we can derive generalization guarantee for (bounded) real-valued loss function using the Rademacher complexity of $\calX$ \cite{bartlett2003rademacher} controlled by the nuclear norm and the degree prior. Our error bound is better than that of \citet{hsieh2014pu} due to the additional constraint of our degree-based prior. \par

\subsection{Learning Node Degree via Cross Validation}
In order to use our node-specific degree prior, we need to estimate the degree of each individual node of the underlying network. The naive strategy of searching the space of all possible $\mathbf{d}$ is clearly infeasible. Thus, we will derive $\mathbf{d}$ from the observed degree $(o_1,o_2,...,o_p)$ where $o_i$ is the degree of node $i$ in the observed network $\mathbf{\Omega}$. Assuming that the network has $K$ edges, under uniform assumption we can estimate the degree of the predicted network by $\mathbf{d}=(d_1,d_2,...,d_p)$ where $d_i=\lceil \frac{2o_i}{\sum_{j=1}^p{o_j}}\times K \rceil$.

Let $F_{\rho}(\X)$ be any loss function for $\X$ where $\rho$ is a hyper-parameter. Considering both the loss and the degree-based prior, we solve the following regularized objective function
\begin{align}
\min_{\X}\  F_{\rho}(\mathbf{X})+\lambda S_{H}(\mathbf{X},c\mathbf{d},\alpha).
  \label{scaled}
\end{align}
where $S_{H}(\mathbf{X},c\mathbf{d},\alpha)$ is the degree-based prior and $c$ can be interpreted as an amplification factor. We define the new degree of node $i$ after multiplying $c$ as $d'_i= \min(\lceil c\times d_i \rceil, p-1)$. 
By setting $c>1$, we amplify the impact of those hub nodes, as the gap of estimated degrees between different hub nodes and non-hub nodes are enlarged, resulting in larger difference in their penalty coefficients in $S_{H}$.
As in the matrix completion setting, we may determine the hyperparameters $\rho$, $\lambda$, $c$, and $\alpha$ through cross validation. That is, we randomly hold out some observed edges as tuning set, and use the remaining observations to train the model and predict all missing edges. We then select the hyperparamters which gives best prediction accuracy on the held-out tuning set.

\subsection{Learning Node Degree via Cross Validation}
In order to use our node-specific degree prior, we need to estimate the degree of each individual node of the underlying network. The naive strategy of searching the space of all possible $\mathbf{d}$ is clearly infeasible. Thus, we will derive $\mathbf{d}$ from the observed degree $(o_1,o_2,...,o_p)$ where $o_i$ is the degree of node $i$ in the observed network $\mathbf{\Omega}$. Assuming that the network has $K$ edges, under uniform assumption we can estimate the degree of the predicted network by $\mathbf{d}=(d_1,d_2,...,d_p)$ where $d_i=\lceil \frac{2o_i}{\sum_{j=1}^p{o_j}}\times K \rceil$.

Let $F_{\rho}(\X)$ be any loss function for $\X$ where $\rho$ is a hyper-parameter. Considering both the loss and the degree-based prior, we solve the following regularized objective function
\begin{align}
\min_{\X}\  F_{\rho}(\mathbf{X})+\lambda S_{H}(\mathbf{X},c\mathbf{d},\alpha).
  \label{scaled}
\end{align}
where $S_{H}(\mathbf{X},c\mathbf{d},\alpha)$ is the degree-based prior and $c$ can be interpreted as an amplification factor. We define the new degree of node $i$ after multiplying $c$ as $d'_i= \min(\lceil c\times d_i \rceil, p-1)$. 
By setting $c>1$, we amplify the impact of those hub nodes, as the gap of estimated degrees between different hub nodes and non-hub nodes are enlarged, resulting in larger difference in their penalty coefficients in $S_{H}$.
As in the matrix completion setting, we may determine the hyperparameters $\rho$, $\lambda$, $c$, and $\alpha$ through cross validation. That is, we randomly hold out some observed edges as tuning set, and use the remaining observations to train the model and predict all missing edges. We then select the hyperparamters which gives best prediction accuracy on the held-out tuning set.

\section{Model and Optimization}  \label{sec:optimization}

\subsection{Matrix Completion with Learned Node-Specific Degree Prior}
We now consider how to solve the network inference problem (i.e., recover $\E$) given some observed edges (i.e., $\mathbf{\Omega}$) using matrix completion regularized by our degree-based prior. Assume that the observed degree of the $p$ variables is $(o_1,o_2,....,o_p)$ and we would like to output $K$ edges as the solution.
Let $\mathbf{d}=(d_1,d_2,...,d_p)$ be the estimated degree of the $p$ nodes in the predicted network where $d_i=\lceil \frac{2o_i}{\sum_{j=1}^p{o_j}}\times K \rceil$. Using non-negative matrix tri-factorization \cite{ding2006orthogonal}, the regularized matrix completion problem can be formulated as
\begin{align} \label{regularized_objective}
\min_{\U\ge 0,\SS\ge 0}\quad & \sum_{ij} \M_{ij}((\U\mathbf{S}{\U}^T)_{ij}-\mathbf{\Omega}_{ij})^2 \nonumber \\
&\quad+\lambda S_H(\U\mathbf{S}{\U}^T,c\mathbf{d},\alpha)
\end{align}
Where $\M_{ij}=1-\frac{\rho}{2}$ if $\mathbf{\Omega}_{ij}=1$ and $\M_{ij}=\frac{\rho}{2}$ otherwise. \par
In the above objective function, the first term is the loss function and the second term is our degree-based prior. Meanwhile, $\lambda$, $c$, $\alpha$ and $\rho$ are the hyper parameters to be tuned through cross-validation. This tri-factorization method has previously been applied to recover protein-protein interaction (PPI) networks~\cite{wang2013predicting}.
Note \eqref{regularized_objective} is different from \eqref{e:obj} 
in that it removes the hard constraints $\mathbf{USU}^T \le 1$ (which is often satisfied by the solution) and it uses factorization instead of nuclear norm constraint to enforce low-rank structure.

\subsection{Optimization}
By introducing additional variables $\X$ and the constraint $\X=\U\mathbf{S}\U^T$, we rewrite \eqref{regularized_objective} as
\begin{align}
\min_{\U\ge 0, \mathbf{S}\ge 0, \X} & \sum_{ij} \M_{ij} ( (\U \mathbf{S} \U^T)_{ij}-\mathbf{\Omega}_{ij} )^2 +\lambda S_H(\X,c\mathbf{d},\alpha) \nonumber \\ \label{new_regularized_objective}
\text{s.t.} & \quad \mathbf{USU}^T = \X.
\end{align}
This formulation can be solved by alternating direction method of multipliers (ADMM, \citealp{boyd2011distributed}), which has been successfully applied to non-negative matrix factorization problems \citep{sun2014alternating}. ADMM solves
\eqref{new_regularized_objective} by iterating the following three steps till convergence, where $\eta>0$ is an optimization parameter we fix
\begin{align}
  \label{Step_One}
& \mathbf{U}^{(t+1)},\mathbf{S}^{(t+1)} = \min_{\mathbf{U}\geq 0,\mathbf{S}\geq 0}\ \sum_{ij} \M_{ij} ( (\U \mathbf{S} \U^T)_{ij}-\mathbf{\Omega}_{ij} )^2 \nonumber \\
& \quad  + \frac{\eta}{2}||\mathbf{USU}^T-\mathbf{X}^{(t)}+\mathbf{Z}^{(t)}||_F^2, \\
  \label{Step_Two}
& \mathbf{X}^{(t+1)} \nonumber =\min_{\mathbf{X}\geq 0}\ \lambda S_H(\mathbf{X},c\mathbf{d},\alpha) \nonumber \\
&\quad +\frac{\eta}{2}||\mathbf{U}^{(t+1)}\mathbf{S}^{(t+1)}({\mathbf{U}^{(t+1)}})^T-\mathbf{X}+\mathbf{Z}^{(t)}||_F^2, \\
& \mathbf{Z}^{(t+1)}=\mathbf{Z}^{(t)}+\mathbf{U}^{(t+1)}\mathbf{S}^{(t+1)}({\mathbf{U}^{(t+1)}})^T-\mathbf{X}^{(t+1)}. \nonumber
\end{align}

\subsubsection{Solving \eqref{Step_One}}
Let $\mathbf{Y}^{(t)}=\mathbf{X}^{(t)}-\mathbf{Z}^{(t)}$ and combine the similar terms in \eqref{Step_One}. Solving \eqref{Step_One} is equivalent to minimizing the function
\begin{align}
  &\hspace*{-2ex}\sum_{ij}{\M_{ij}((\mathbf{USU}^T)_{ij}-\mathbf{\Omega}_{ij})^2}+\frac{\eta}{2}\sum_{ij}{((\mathbf{USU}^T)_{ij}-\mathbf{Y}^{(t)}_{ij})^2} \nonumber \\
  &\hspace*{-0ex}=\sum_{ij}{(\M_{ij}+\frac{\eta}{2})(\mathbf{USU}^T)_{ij}^2} \nonumber \\   
&\hspace{2ex} -\sum_{ij}{(2\M_{ij}\mathbf{\Omega}_{ij}+\eta \mathbf{Y}_{ij}^{(t)})(\mathbf{USU}^T)_{ij}} + const,  \nonumber
\end{align}
which can be further reduced to solving
\begin{align} \nonumber
  \min_{\mathbf{U}\geq 0,\mathbf{S}\geq 0} \sum_{ij} (\M_{ij}+\frac{\eta}{2}) \left( (\U\SS\U^T)_{ij}- \frac{\M_{ij}\Oij+\frac{\eta}{2}\mathbf{Y}_{ij}^{(t)}}{\M_{ij}+\frac{\eta}{2}}\right)^2.
\end{align}
This is a weighted matrix tri-factorization problem which we solve with the algorithm of \citet{ding2006orthogonal}.

\subsubsection{Solving \eqref{Step_Two}}
We can  divide \eqref{Step_Two} into $p$ subproblems, one for each node of the graph. The objective for the $i$-th node is
\begin{align}
  \label{small_problem}
  \min_{\mathbf{X}_i\geq 0}\ \frac{\eta}{2} \norm{\mathbf{X}_i-\mathbf{A}_i}^2+\lambda\sum_{k=1}^{p-1}{b_i(k) \X_{i,[k]}},
\end{align}
where $\mathbf{A}=\mathbf{U}^{(t+1)}\mathbf{S}^{(t+1)}({\mathbf{U}^{(t+1)}})^T+\mathbf{Z}^{(t)}$, $\mathbf{X}_i$ and $\mathbf{A}_i$ are the $i^{th}$ column of $\A$ and $\X$ respectively, and $b_i(k)=\frac{H^{\alpha}_k}{H^{\alpha}_{cd_i}}$ for $1\leq i < p$. However, these $p$ smaller problems are not independent since $\X$ needs to be symmetric.
To deal with this constraint, we apply the idea of \citet{defazio2012convex} and rewrite the problem as 
\begin{align}
  \label{dual_relax}
  \min_{\mathbf{X}\geq 0}\ &{\frac{\eta}{2} \norm{\mathbf{X}-\mathbf{A}}_F^2+\lambda S_H(\mathbf{X},c\mathbf{d},\alpha)}  \\
  & \text{s.t.} \quad \mathbf{X}=\mathbf{X}^T. \nonumber
\end{align}
Then we can apply dual decomposition \cite{sontag2011introduction} to \eqref{dual_relax}. Specifically, we introduce a Lagrangian term $\trace{ \B (\mathbf{X}-\mathbf{X}^T) }$ 
and minimize the following objective function in each iteration of dual decomposition
\begin{align}
  \label{dual_iteration}
  \min_{\mathbf{X}\geq 0}\ {\frac{\eta}{2} \norm{ \mathbf{X}-\mathbf{A} }_F^2+\lambda S_H(\mathbf{X},c\mathbf{d},\alpha)+ \trace{(\B-\B^T) \mathbf{X}} },
\end{align}
where $\B\in \bbR^{p\times p}$ and its entries are adjusted according to the difference between $\mathbf{X}$ and $\mathbf{X}^T$ 
in each iteration. By completing the squares, \eqref{dual_iteration} is transformed into the form
\begin{align}
  \min_{\mathbf{X}\geq 0}\ &{\frac{\eta}{2} \norm{\mathbf{X}-\A'}_F^2+\lambda S_H(\mathbf{X},c\mathbf{d},\alpha)+const}.
\end{align}
This objective can now be decomposed into $p$ independent subproblems, each of which has the same form of \eqref{small_problem}. We solve each subproblem using the algorithm of \citet{tang2015learning,bogdan2013statistical} with time complexity $O(p\log p)$.

\section{Related Work} \label{sec:related}
There have been several work on applying matrix completion techniques to the link prediction task, including \citet{hsieh2014pu,wang2013predicting,huang2013robust}. \citet{hsieh2014pu} have considered the PU (positive-unlabeled) learning setting for matrix completion, where the observed entries are purely $1$s while all other entries are unlabeled. 
\citet{wang2013predicting} have applied the orthogonal matrix tri-factorization technique \cite{ding2006orthogonal} to the problem of protein-protein interaction prediction (a real biological problem that can be modeled as link prediction), and yielded significant improvement on prediction accuracy. \citet{huang2013robust} have used a trace norm regularized discrete matrix completion to predict new links in social networks and protein-protein interaction network. However, all these matrix completion based approaches does not carefully utilize the degree information conveyed by the observed samples.

In this work, we have used a degree prior regularized matrix completion framework for predicting missing edges of a network. 
Although our degree prior may be mathematically similar to some existing scale-free priors in the literature \cite{liu2011learning,defazio2012convex,tang2015learning}, there exist clear difference between our work and theirs. Existing priors are all based upon the (global) scale-free assumption, while the prior we use here is directly learned from observed samples. Both theoretical and experimental results show that our learned prior leads to better prediction performance in practice. Furthermore, our work addresses a very different problem setting than the abovementioned previous works do. While those previous works estimate the network from a given covariance matrix calculated from observed attributes of nodes (e.g., gene expression data), our work aims to predict missing links in a partially observed network without any observed attributes of network nodes. For example, those works are suitable for gene coexpression network inference from a set of measured gene expression levels, but not for missing link prediction in a social network. The latter problem can be attacked using the method proposed here.

\section{Experimental Results} \label{sec:results}
We have implemented the proposed method described in section \ref{sec:optimization} (denoted as Tri+Degree) and then compare its performance against the following 4 methods: tri-factorization method without any prior (denoted as Tri), tri-factorization regularized by $l_1$ penalty (denoted as Tri+L1), tri-factorization regularized by scale-free penalty (denoted as Tri+Scale) \cite{liu2011learning} and PU learning for matrix completion (denoted as PU) \cite{hsieh2014pu}. Cross validation is performed for all the methods for hyper parameter selection.
 We report for each method the averaged result over 5 random seeds (used for the initialization of matrix completion).

We conduct three experiments to test the performance of the 5 methods.
In the first experiment, we use a simulated gene co-expression network in the well-known DREAM5 challenge \cite{marbach2012wisdom}.
The second experiment makes use of one protein-protein interaction (PPI) network in BioGrid \cite{stark2006biogrid}.
We randomly sample some of the known edges in this PPI network as observations and use the unsampled edges as the test set.
In the last experiment, we test different methods using multiple releases of the protein-protein interaction network for the species ``Rat''. In particular, we predict the network from an old release and then evaluate the consistency between the predicted network and a newer release.
Our experiments show that our method obtains the best performance in almost all the experimental settings and in some cases, our method shows significant advantage over the others.

\begin{figure*}[!htb]
  \centering
\psfrag{sampling rate=0.9}{}
\psfrag{sampling rate=0.7}{}
  \begin{tabular}{@{}c@{\hspace{0.01\linewidth}}c@{\hspace{0\linewidth}}c@{\hspace{0\linewidth}}c@{}}
& uniform sampling & over-sampling & under-sampling \\[2ex]
\rotatebox{90}{\hspace{3em}sampling rate=90\%} &
\psfrag{percentage}{}
\psfrag{predict}[bc][c]{\# correctly predicted edges}
    \includegraphics[width=0.32\textwidth]{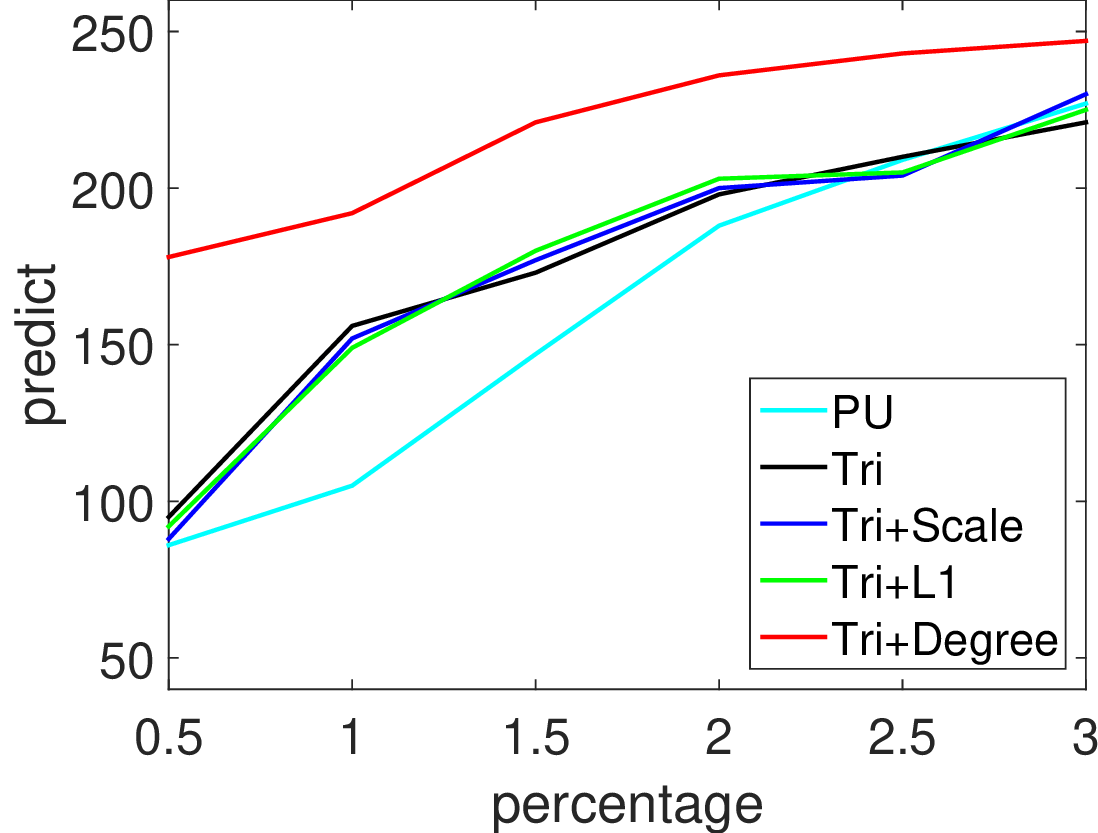} &
\psfrag{percentage}{}
\psfrag{predict}{}
     \includegraphics[width=0.32\textwidth]{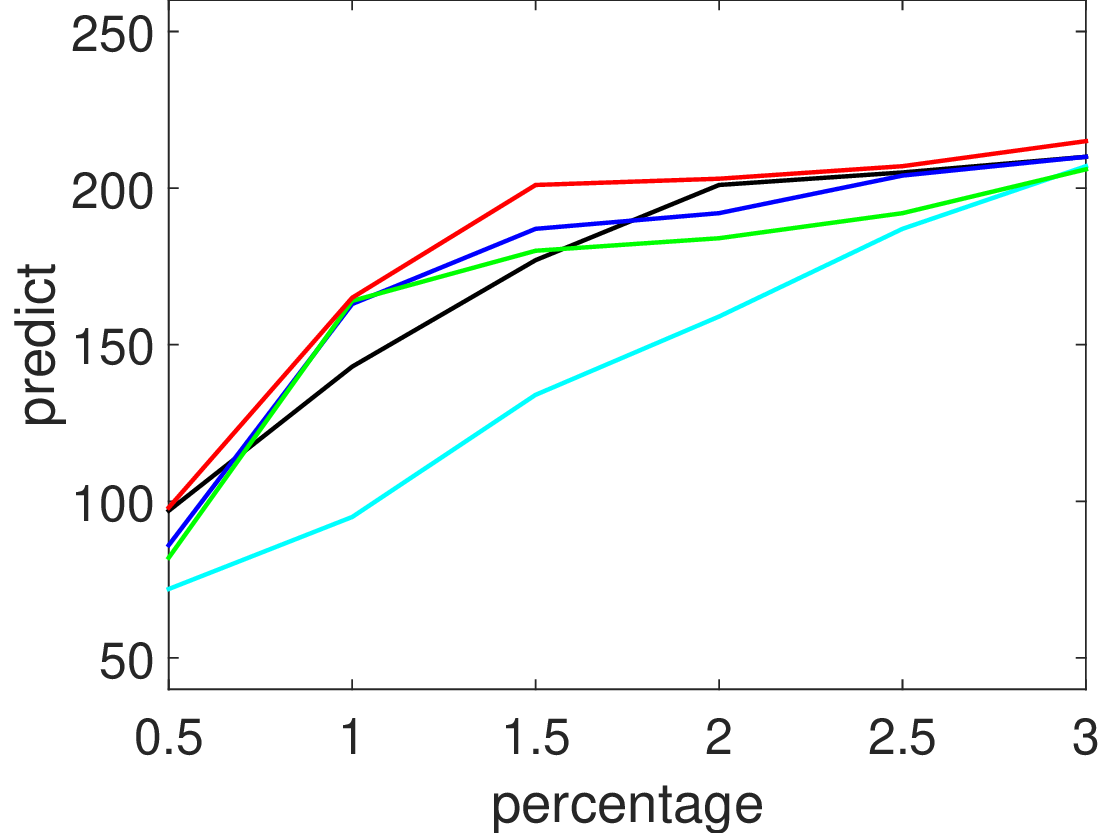} &
\psfrag{percentage}{}
\psfrag{predict}{}
      \includegraphics[width=0.32\textwidth]{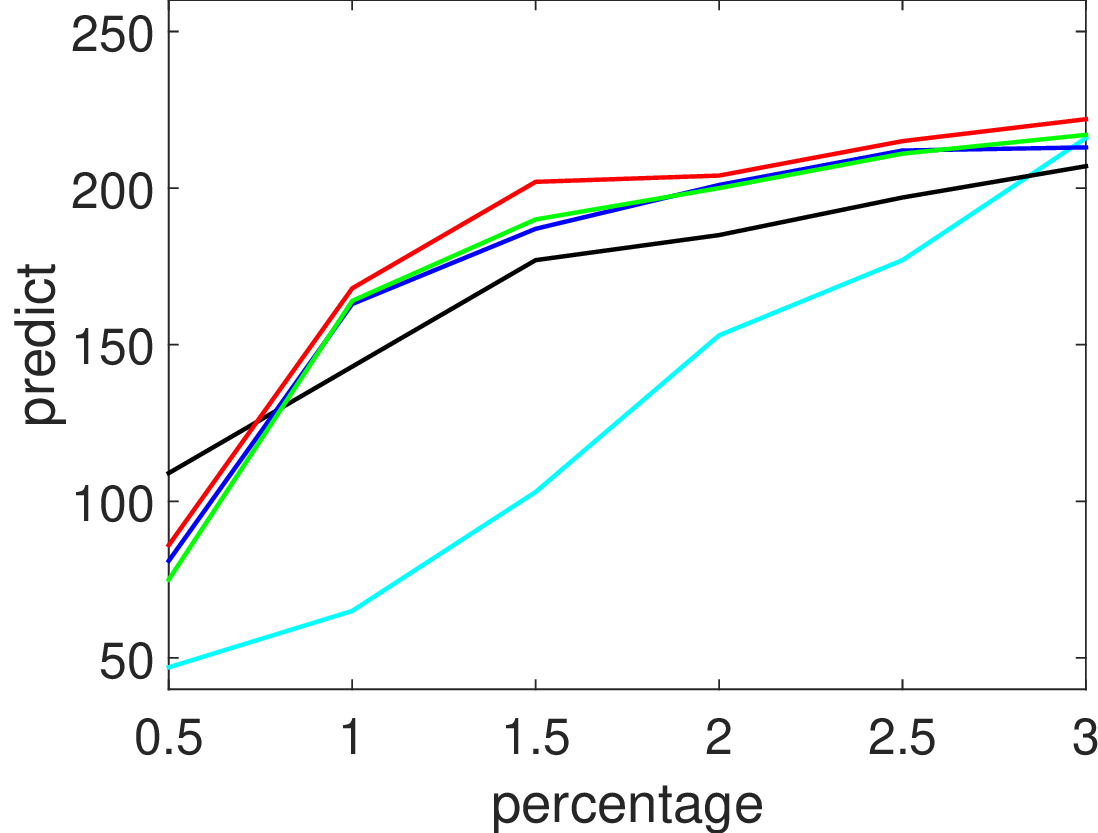} \\
\rotatebox{90}{\hspace{3em}sampling rate=70\%} & 
\psfrag{percentage}[c][c]{percentage of all pairs (\%)}
\psfrag{predict}[bc][c]{\# correctly predicted edges}
    \includegraphics[width=0.32\textwidth]{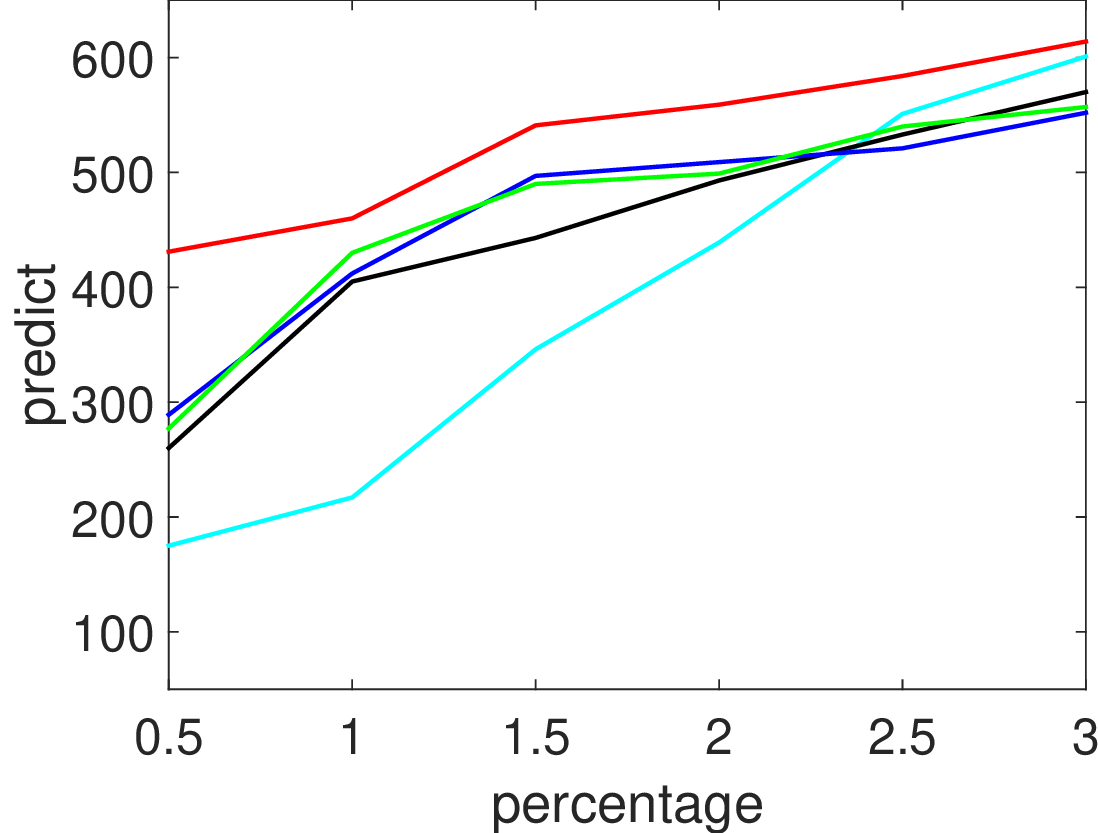} &
\psfrag{percentage}[c][c]{percentage of all pairs (\%)}
\psfrag{predict}{}
    \includegraphics[width=0.32\textwidth]{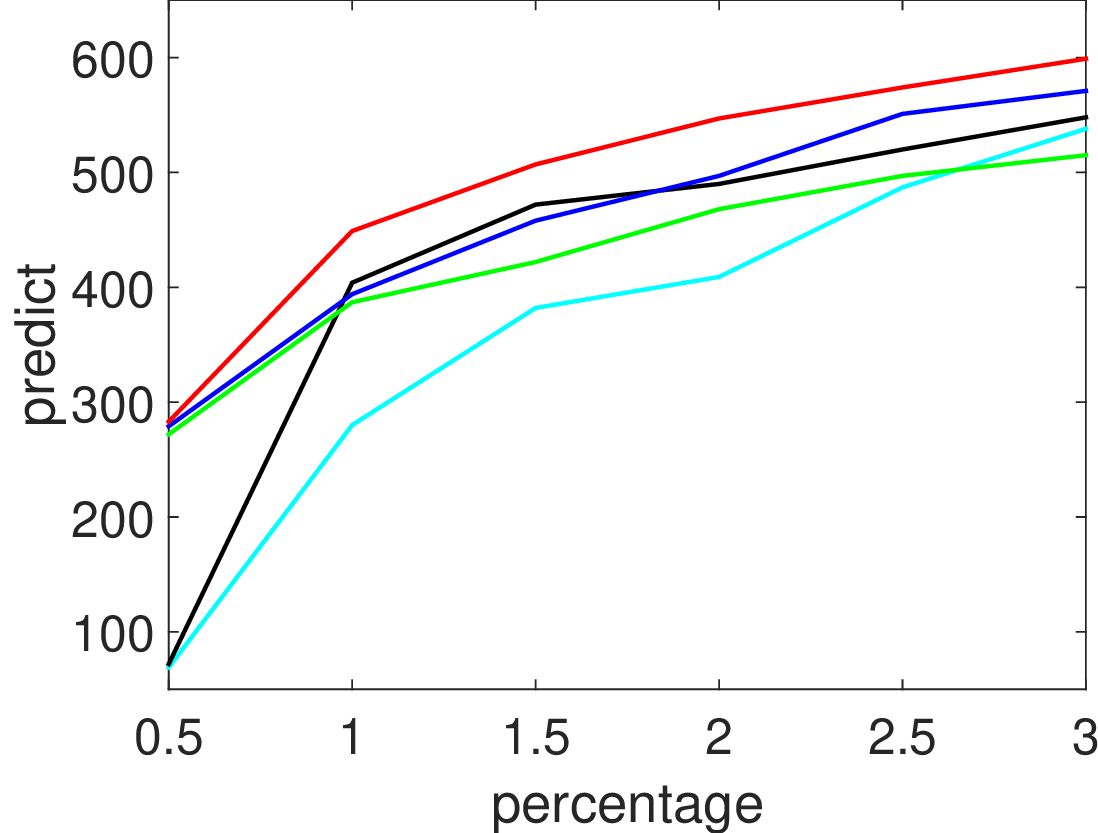} &
\psfrag{percentage}[c][c]{percentage of all pairs (\%)}
\psfrag{predict}{}
    \includegraphics[width=0.32\textwidth]{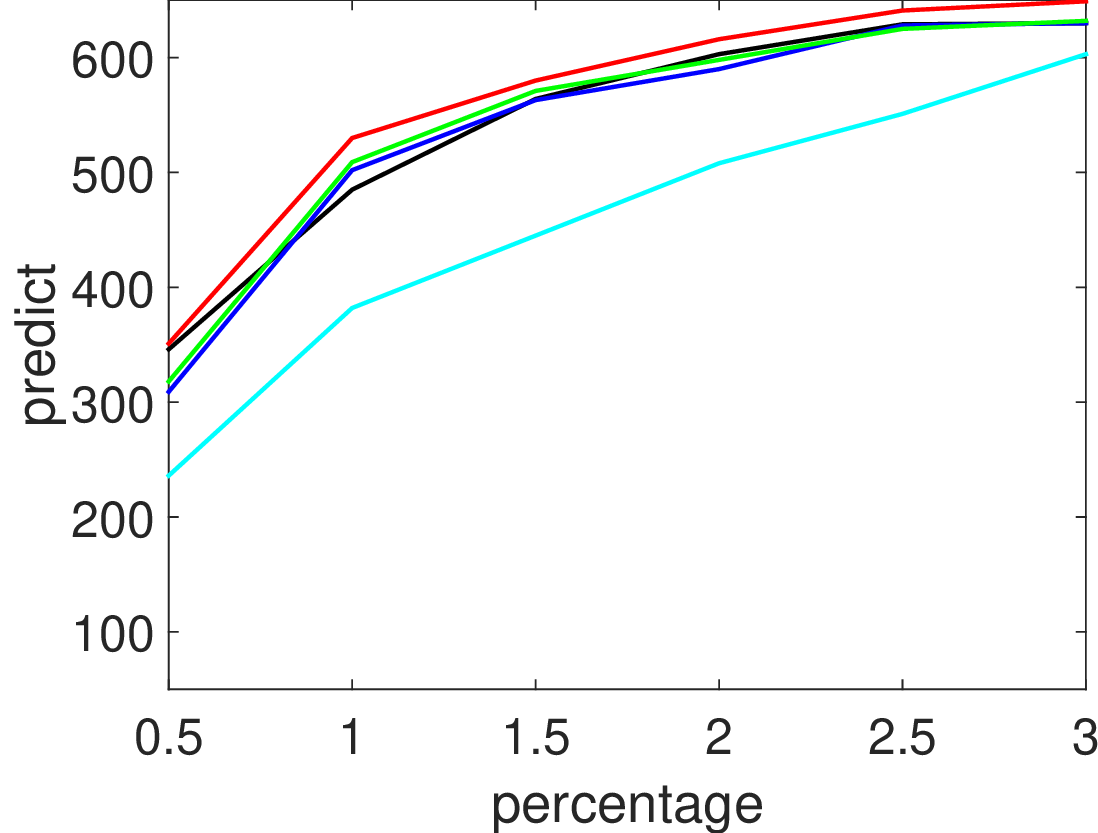} \\
  \end{tabular}
  \caption{Performance of the 5 methods on a simulated gene co-expression network. $x$-axis shows the percentage of all possible pairs ($\frac{p(p-1)}{2}$) being predicted, and $y$-axis shows the number of correctly predicted edges (excluding already observed edges).}
  \label{gene}
\end{figure*}

\begin{figure*}[!htb]
  \centering
\begin{tabular}{@{}c@{\hspace{0.0\linewidth}}c@{}}
sampling rate=90\% & sampling rate=70\% \\[1.5ex]
\psfrag{sampling rate=0.9}{}
\psfrag{percentage}[c][c]{percentage of all pairs (\%)}
\psfrag{predict}[c][c]{\# correctly predicted edges}
  \includegraphics[width=0.45\textwidth]{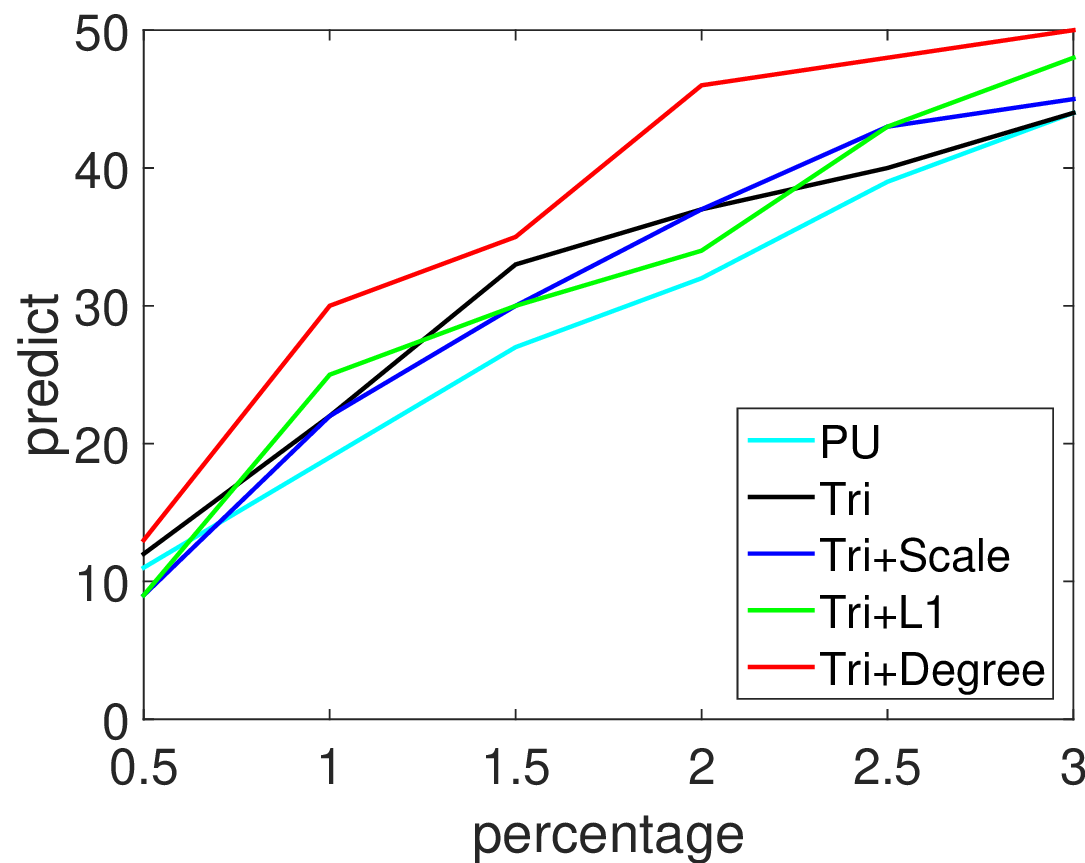}& 
\psfrag{sampling rate=0.7}[c][c]{}
\psfrag{percentage}[c][c]{percentage of all pairs (\%)}
\psfrag{predict}[c][c]{}
  \includegraphics[width=0.45\textwidth]{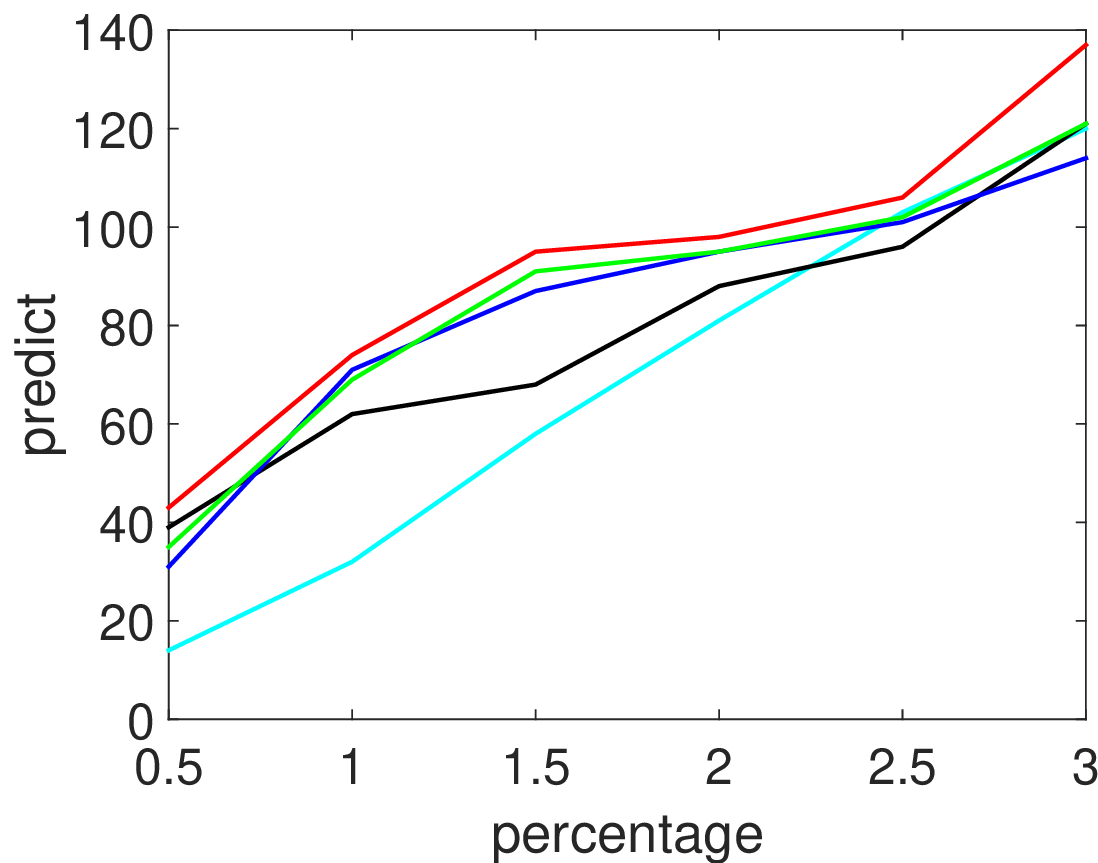}
\end{tabular}
  \caption{Performance of the 5 methods on a real protein-protein interaction network. $x$-axis shows the percentage of all possible pairs ($\frac{p(p-1)}{2}$) being predicted, and $y$-axis shows the number of correctly predicted edges (excluding already observed edges).}
  \label{BioGridPPI}
\end{figure*}

\begin{table*}[t]
\caption{Test results when $0.05\times \frac{p(p-1)}{2}$ new edges are predicted where $p$ is the number of proteins of the older release. From left to right, each entry contains the ratios of the edges only in a newer release being recovered by the following methods: PU, Tri, Tri+Scale, Tri+L1 and our method Tri+Degree.}
\begin{center}
	\begin{tabular}{l | c | c }
	- & 3.2.104 & 3.2.114 \\
	\hline
	3.1.94 & 0.3269 / 0.3075 / 0.2105 / 0.3324 / \textbf{0.39893} & 0.2944 / 0.3318 / 0.1916 / 0.3154 / \textbf{0.4112}  \\
	\hline
	3.2.104 & - & 0.2360 / 0.2360 / 0.3483 / 0.3820 / \textbf{0.5169} \\
	\end{tabular}
\end{center}
\label{tbl:multiReleases}
\end{table*}

\subsection{Gene Expression Network}
The ground truth network has $1565$ nodes and $3996$ edges.
We consider a node to be a hub if its degree is among top 20\% of all nodes.
To conduct the experiment, we sample some edges from the ground truth as observations. The following three different sampling conditions are tested. 
\begin{itemize}
\item \textbf{Uniform sampling:} All edges are randomly and uniformly sampled. Two different sampling rates $90\%$ and $70\%$ are used.
\item \textbf{Over-sampling:} Edges adjacent to hubs are over-sampled. That is, we sample $95\%$ of edges adjacent to hubs and $90\%$ of edges not adjacent to hubs. In another setting, we sample $75\%$ edges adjacent to hubs and $70\%$ of edges not adjacent to hubs.
\item \textbf{Under-sampling:} Edges adjacent to hubs are under-sampled. Two different settings are tested: $85\%$ vs $90\%$ and $65\%$ vs $70\%$ for edges adjacent to hubs vs. edges not adjacent to hubs, respectively.
\end{itemize}
The second and third strategies are used to test how robust our method is with respect to sampling bias.

We give the results obtained by different methods in Figure \ref{gene}, where each figure shows how many predictions agree with the ground truth as we are predicting more and more edges (increasing $K$) from the real-valued matrix. (We exclude the already observed edges for counting the ``number of correctly predicted edges'', so the max number of correct prediction for the 90\% sampling rate is 1/3 of that for the 70\% sampling rate.) Figures in the top row show results for the base sampling rate of 90\%, and figures in the bottom row show results for the more challenging base sampling rate of 70\%. The left column of Figure \ref{gene} shows the experimental results when the observed edges are uniformly sampled from true edges, which is the usual assumption in many real applications. In this case, the performance of our approach is much better than the others. This is because the degree distribution learned from the observation is consistent with the ground truth, and thus our degree prior is most effective. Further, a larger sampling rate results in a better performance than a smaller sampling rate because when more edges are observed, the learned degree distribution is closer to the ground truth. 
The middle column and the right column of Fig. \ref{gene} show results under the over-sampling and under-sampling conditions, respectively. In these two cases, our approach Tri+Degree still outperforms the others, though not as much as in the uniform sampling condition. This shows the robustness of our method with respect to the mismatch in sampling conditions. Overall, our method achieves the best prediction accuracy in almost all conditions for a range of $K$ values.

\subsection{Protein-Protein Interaction Network}
We further use the protein-protein interaction (PPI) network for Plasmodium in the BioGrid database (release 3.3.124) to evaluate the performance of the five methods. This network has $1227$ nodes and $2445$ experimentally-validated true edges. We sample $90\%$ and $70\%$ of the true edges as the input of all the tested methods, respectively. As shown in Figure \ref{BioGridPPI}, our method Tri+Degree outperforms the others in both settings. This shows that, in addition to gene co-expression networks, our method also works well on PPI networks. \par

\subsection{Multiple Releases of Protein-Protein Interaction Networks}
In the above two experiments, the training and test edges are sampled from the networks to be inferred. Here we test our method with another strategy. In particular, we make use of three releases $3.1.94$, $3.2.104$ and $3.2.114$ of the protein-protein interaction (PPI) network for the Rat species, which are taken from the BioGrid database. The three releases have $1568$, $2651$ and $2810$ nodes, respectively, and  $1813$, $3499$ and $3769$ known edges, respectively. A newer release typically contains all the edges in an older release.
In this experiment, we predict PPIs from an older release and then check if the predicted PPIs appear in a newer release or not. A prediction method is better if it yields a larger number of predicted PPIs in a newer release.\par
Table \ref{tbl:multiReleases} shows that our method outperforms the other four methods significantly. Each entry in the table contains the test results of $5$ methods on two different releases. For example, the entry at the third row and the third column contains the results obtained by using release $3.2.104$ as input and the difference between releases $3.3.114$ and $3.2.104$ as the test set. In each entry, from left to right the numbers are the ratios of the edges only in a newer release being recovered by PU, Tri, Tri+Scale, Tri+L1 and our method, respectively. 
Each method predicts $0.05\times \frac{p(p-1)}{2}$ new edges where $p$ is the number of proteins in the older release. In fact, when predicting more edges, our method has a even larger advantage over the others.

\section{Conclusion}\label{sec:conclusion}
This paper presents a novel method for network inference (link prediction) using a node-specific degree prior learned from a subset of observed edges. We show theoretically that the proposed degree prior is effective in inducing a network that approximately follows the observed degree distribution. We propose to use a matrix completion objective regularized by our degree prior for network inference, and provide the recovery guarantee of our structured matrix completion method under the uniform sampling assumption. Our experimental results show that our method achieves better performance than existing approaches on both simulated gene-expression networks and real protein-protein interaction networks. \newline

Our analysis have mostly assumed that the observed edges are uniformly sampled from the true network, which is typically the case in many real applications. 
It turns out that our approach also performs better than existing methods when the adjacent edges of hubs are somewhat upsampled/downsampled, 
so our method is not very sensitive to mismatch in sampling conditions. However, it is interesting and useful to study the performance of different methods under more challenging conditions. 

\onecolumn

\section{Proof of Theorem~2.2}

Mathematically, the network inference problem addressed in our paper is related to PU (Positive-Unlabeled) learning for matrix completion \cite{hsieh2014pu}. Thus, we start from the same problem setting as \cite{hsieh2014pu} and borrows some similar techniques in proving our bound. We theoretically show that our degree prior does improve the network inference as it guarantees a better bound of recovery error. \par 
\paragraph{Problem setting} We assume that a 0/1 matrix $\E$ is observed from a real-valued matrix $\X^*$ by a thresholding process: Let $q\in \bbR$ be the threshold value, then $\Eij= thr(\X^*_{ij}):=\ind{ \X^*_{ij} > q }$ where $\ind{\cdot}$ is the indicator function. Furthermore, we assume that a subset of the edge set $E=\{(i,j): \Eij=1\}$ is observed by uniform sampling elements from $E$ with probability $\rho\in (0,1)$. We use the matrix $\bOmega$ to denote the observations, i.e., $\Oij=1$ if the edge $(i,j)$ is observed and $\Oij=0$ otherwise. This data generation process indicates
\begin{eqnarray*}
\condP{\Oij=1}{\Eij=1}=\rho, &\condP{\Oij=0}{\Eij=1}=1-\rho, \\
\condP{\Oij=1}{\Eij=0}=0, &\condP{\Oij=0}{\Eij=0}=1.
\end{eqnarray*}

Consider the following objective as the proxy for recovering $\E$:
\begin{gather} \label{e:objapp}
 \min\limits_{\X} \ R_{l}(\X;\bOmega) := \left( 1 - \frac{\rho}{2} \right) \sum_{(i,j):\Oij=1}\; \left(\Xij-1 \right)^2  + \frac{\rho}{2} \sum_{(i,j):\Oij=0} \left(\Xij-0\right)^2 \\
\text{s.t.}\quad \X \in \calX:=\{ \X\in \bbR^{p\times p} \;\mid\; \X=\X^\top, \quad 0\le \X \le 1, \quad \norm{\X}_*\le t, \quad S_{H}(\X,d,\alpha)\le r \}, \nonumber
\end{gather}
where $0\le \X \le 1$ indicates elementwise comparison, i.e.,  $0\le \Xij \le 1$, $\forall 1\le i,j \le p$. This objective assigns different weights to the squared error between $\Xij$ and $\Oij$ depending on the observed value $\Oij$. When the percentage of unobserved edges is high ($\rho$ is small), the objective emphasizes more on predicting the observed $1$'s and allows for larger predictions for unobserved entries; otherwise it tends to predict smaller values (close to $0$) for the unobserved entries.

We assume the underlying real-valued matrix $\X^*$ comes from the class $\calX$. Once we obtain the minimizer $\hat{\X}$ of \eqref{e:objapp}, we recover a binary matrix $\bar{\X}$ using $\hat{\X}$ by thresholding, i.e., $\bar{\X}_{ij}=\ind{\hat{\X}_{ij}>q}$.

We define the \emph{expected recovery error} of an estimator $\X(\bOmega)$ based on observation $\bOmega$ as
\begin{align*}
R(\X)=\ExpO{ \sum_{i,j} \ind{ thr(\Xij)\neq \Eij } },
\end{align*}
where the expectation is taken over random selection of observed edges (i.e., the randomness in observing $\bOmega$ from $\E$).

In order to relate our objective to the recovery error, we will make use of the following \emph{label-dependent error}
\begin{align*}
R_{\rho} (\Xij,\Oij) & = \left( 1-\frac{\rho}{2}\right) \ind{thr(\Xij)=0} \ind{\Oij=1}\\
& + \frac{\rho}{2} \ind{thr(\Xij)=1} \ind{\Oij=0}, \\
R_{\rho} (\X;\bOmega) & = \sum_{i,j} R_\rho (\Xij,\Oij).
\end{align*}
The \emph{expected label-dependent error} of an estimator $\X(\bOmega)$ is defined as
\begin{align}
R_{\rho} (\X) : = \ExpO{R_{\rho} (\X;\bOmega)}=\ExpO{\sum_{i,j} R_\rho (\Xij,\Oij) }.
\end{align}

The following lemma shows that the expected label-dependent error is proportional to the expected recovery error.
\begin{lemma}\label{lem:recovery-label}
Let $\X(\bOmega)$ be an estimator based on observations $\bOmega$. Then there exists some constant $b$ independent of $\X$, such that
\begin{align}
R_{\rho} (\X) = \frac{\rho}{2} R (\X) + b.
\end{align}
\end{lemma}
\begin{proof}
For notational symplicity, we denote by $\tXij=thr(\Xij)\in \{0,1\}$ the thresholded binary value from $\Xij$, and by $\tX$ the entire thresholded 0/1 matrix.

Define $\eta(\tXij)=\condP{\Eij=1}{\tXij}$ and $\xi(\tXij)=\condP{\Oij=1}{\tXij}$. According to the data generation process, we have
\begin{align}
\xi(\tXij) & = \condP{\Oij=1, \Eij=1}{\tXij} +
\condP{\Oij=1, \Eij=0}{\tXij} \nonumber \\
& = \condP{\Eij=1}{\tXij} \condP{\Oij=1}{\Eij=1} +
\condP{\Eij=0}{\tXij} \condP{\Oij=1}{\Eij=0} \nonumber \\
& = \rho\cdot \eta(\tXij). \label{e:Pthres}
\end{align}

Consider now the two cases in which $\tXij$ makes an error in predicting $\Oij$.

\noindent\textbf{(i)} When $\tXij=0$ and $\Oij=1$, according to \eqref{e:Pthres} we have
\begin{align*}
\condP{\Oij=1}{\tXij=0} = \xi(0) = \rho\cdot \eta(0).
\end{align*}

\noindent\textbf{(ii)} When $\tXij=1$ and $\Oij=0$, according to \eqref{e:Pthres} we have
\begin{align*}
\condP{\Oij=0}{\tXij=1} = 1-\xi(1) = 1-\rho\cdot \eta(1).
\end{align*}

Since $\tX$ is also uniquely determined by the estimator for fixed $\bOmega$, we can instead compute $R_{\rho}(\X)$ by taking expectation over the thresholded matrix $\tX$ as follows
\begin{align*}
R_{\rho}(\X)&  = \ExpX{ \sum_{ij} \left(1-\frac{\rho}{2}\right) \condP{\Oij=1}{\tXij=0} \ind{\tXij=0} +  \frac{\rho}{2} \condP{\Oij=0}{\tXij=1} \ind{\tXij=1} }  \\
 & = \ExpX{ \sum_{ij} \left(1-\frac{\rho}{2}\right) \xi(0) \ind{\tXij=0} +  \frac{\rho}{2} \left(1-\xi(1)\right) \ind{\tXij=1} }  \\
& = \ExpX{ \sum_{ij} \left(\rho-\frac{\rho^2}{2}\right) \eta(0) \ind{\tXij=0} + \frac{\rho}{2} \left(1- \rho\cdot \eta(1)\right) \ind{\tXij=1} } \\
& = \bbE_{\tX} \left[ \sum_{ij} \frac{\rho}{2} \eta(0) \ind{\tXij=0} + \frac{\rho-\rho^2}{2} \eta(0) \ind{\tXij=0} \right. \\
& \qquad\; \left. + \frac{\rho}{2} \left(1-\eta(1)\right) \ind{\tXij=1} + \frac{\rho-\rho^2}{2} \eta(1) \ind{\tXij=1} \right] \\
& = \ExpX{ \sum_{ij} \frac{\rho}{2} \eta(0) \ind{\tXij=0} + \frac{\rho}{2} \left(1-\eta(1)\right) \ind{\tXij=1} + \frac{\rho-\rho^2}{2} \eta(\tXij) } \\
& = \ExpX{ \sum_{ij} \frac{\rho}{2} \condP{\Eij=1}{\tXij=0} \ind{\tXij=0} + \frac{\rho}{2}\condP{\Eij=0}{\tXij=1} \ind{\tXij=1} + \frac{\rho-\rho^2}{2} \eta(\tXij) } \\
& = \frac{\rho}{2} R(\X) + \ExpX{ \frac{\rho-\rho^2}{2} \eta(\tXij) }
\end{align*}
We conclude the proof by setting $b=\ExpX{ \frac{\rho-\rho^2}{2} \eta(\tXij) }$ which is independent of $\X$ due to the expectation.
\end{proof}

We now connect the label-dependent error to the weighted quadratic objective.
\begin{lemma} \label{lem:label-loss}
For any $q$, we have
\begin{align} \label{e:loss-and-error-weak}
 R_\rho(\X) - \min_{\X \in \calX}\ R_\rho(\X) \le \gamma' R_{l}(\X)
\end{align}
where $\gamma'=\max\left(\frac{1}{q^2},\frac{1}{(2-\rho)(1-q)^2}\right)$.
\end{lemma}
Furthermore, if $q < \frac{2-\rho}{3-2\rho}$, we have
\begin{align} \label{e:loss-and-error}
 R_\rho(\X) - \min_{\X \in \calX}\ R_\rho(\X) \le \gamma \left( R_{l}(\X) - \min_{\X}\ R_{l}(\X) \right)
\end{align}
where $\gamma=\max\left(\frac{1}{q^2}, \frac{1}{(3-2\rho)\left(q - \frac{2-\rho}{3-2\rho} \right)^2} \right)$.
\begin{proof}
First, we define $R_{l}(\X)$ as
\begin{align}
& R_{l} (\X) : = \ExpO{R_{l} (\X;\bOmega)}=\ExpO{\sum_{i,j} R_l (\Xij,\Oij) } \nonumber
\end{align}

Let us first prove \eqref{e:loss-and-error-weak}. Consider the following two cases.

\noindent \textbf{(i)} If $\Eij=0$, then $\Oij=0$ and the two types of losses incurred at entry $(i,j)$ are
\begin{align*}
R_{\rho}(\Xij)&= \frac{\rho}{2} \ind{\Xij > q},\qquad \min_{\Xij}\; R_{\rho}(\Xij)=0 \\
R_{l}(\Xij)&=  \frac{\rho}{2} \Xij^2,\;\quad\qquad\qquad \min_{\Xij}\; R_{l}(\Xij)=0.
\end{align*}
If $\Xij\le q$, then $LHS=0$ and $RHS\ge 0$, so \eqref{e:loss-and-error-weak} holds at entry $(i,j)$ trivially; otherwise $\Xij > q$, then $LHS=\frac{\rho}{2}$, and $RHS\ge \gamma' \frac{\rho q^2}{2}$, so \eqref{e:loss-and-error-weak} holds at entry $(i,j)$ with $\gamma'=\frac{1}{q^2}$.

\noindent \textbf{(ii)} If $\Eij=1$, then $\condP{\Oij=1}{\Eij=1}=\rho$ and the two types of losses incurred at entry $(i,j)$ are
\begin{align*}
R_{\rho}(\Xij)&=\frac{\rho(2-\rho)}{2}\ind{\Xij\le q}+\frac{\rho(1-\rho)}{2}\ind{\Xij > q},\quad \min_{\Xij}\; R_{\rho}(\Xij)=\frac{\rho(1-\rho)}{2},  \\
R_{l}(\Xij)&=\frac{\rho(2-\rho)}{2} (\Xij-1)^2 + \frac{\rho(1-\rho)}{2} \Xij^2.
\end{align*}
If $\Xij>q$, we have $LHS=0$ and \eqref{e:loss-and-error-weak}  holds at entry $(i,j)$ trivially; otherwise $\Xij \le q$ and $LHS=\frac{\rho}{2}$, while
\begin{align*}
RHS\ge \gamma' \frac{\rho(2-\rho)}{2} (1-q)^2,
\end{align*}
so \eqref{e:loss-and-error-weak} holds at entry $(i,j)$ with $\gamma'=\frac{1}{(2-\rho)(1-q)^2}$.
Summing over all $(i,j)$ concludes the proof for \eqref{e:loss-and-error-weak}.

To prove \eqref{e:loss-and-error} we only need to consider $\min_{\Xij}\; R_{l}(\Xij)$
for $\Xij\le q$, as the cases \textbf{(i)} and case \textbf{(ii)} with $\Xij>q$ continue to hold. By taking the derivative of $R_{l}(\Xij)$ and setting it to zero, we observe that the minimum of $\min_{\Xij\in\bbR}\; R_{l}(\Xij)$ is achieved at $\frac{2-\rho}{3-2\rho}\in \left( \frac{2}{3}, 1 \right)$, with $\min_{\Xij}\ R_{l}(\Xij)=\frac{\rho(1-\rho)(2-\rho)}{2(3-2\rho)}$. But since we require $q<\frac{2-\rho}{3-2\rho}$, for $\Xij\le q$ we have
\begin{align*}
R_{l}(\Xij) - \min_{\Xij}\ R_{l}(\Xij) \ge R_{l}(q)-R_l\left(\frac{2-\rho}{3-2\rho}\right) = \frac{\rho (3-2\rho)}{2} \left(q - \frac{2-\rho}{3-2\rho} \right)^2 > 0.
\end{align*}
It is then clear that we can ensure \eqref{e:loss-and-error} in this case by setting $\eta=\frac{1}{(3-2\rho)\left(q - \frac{2-\rho}{3-2\rho} \right)^2}$. Combining all cases gives the desired inequality.
\end{proof}

\textbf{Theorem 2.2} Assume the underlying graph has a degree distribution $\mathbf{d}^* =(d_1^*,\dots,d_p^*)$ with a total degree $s=2 \abs{E}=\sum_{i=1}^p d_p^*$ and maximum degree $d_{\max}^*=\max\ (d_1^*,\dots,d_p^*)$, and  assume the threshold value used to obtain $\E$ from the underlying real matrix $\X^*$ is $q < \frac{2-\rho}{3-2\rho}$. Let $\hat{\X}(\bOmega)$ be the minimizer of the weighted quadratic objective \eqref{e:objapp}, where the regularizer $S_H(\X,\mathbf{d},\alpha)$ uses an estimated degree distribution $\mathbf{d}=(d_1,\dots,d_p)$ with maximum degree $d_{\max}=\max\ (d_1,\dots,d_p)$. Then with probability at least $1-\delta$, we have
\begin{align*}
R(\hat{\X}) \le \frac{4\gamma (2-\rho)}{\rho}  \left(
2 \min\left( t C ( 2 \sqrt{d_{\max}^*} + \sqrt[4]{s}),\ r \log_2^\alpha (d_{\max}+1)  \right) + \sqrt{\frac{s \log 2/\delta}{2}}
\right)
\end{align*}
where $\gamma=\max\left(\frac{1}{q^2}, \frac{1}{(3-2\rho)\left(q - \frac{2-\rho}{3-2\rho} \right)^2} \right)$ and $C$ is a universal constant.

\begin{proof}
In view of Lemma~\ref{lem:recovery-label}, we observe that
\begin{align*}
R(\hat{\X}) - \min_{\X} R(\X) = \frac{2}{\rho} \left( R_{\rho} (\hat{\X}) - \min_{\X} R_{\rho} (\X) \right).
\end{align*}
Notice $\min_{\X} R(\X)=R(\X^*)=0$ according to the data genearation process. Now apply Lemma~\ref{lem:label-loss} and we obtain
\begin{align*}
R(\hat{\X}) \le \frac{2\gamma}{\rho} \left( R_{l}(\hat{\X}) - \min_{\X}\ R_{l}(\X) \right).
\end{align*}
Notice $R_{l}(\hat{\X}) - \min_{\X}\ R_{l}(\X)$ is the sub-optimality of the empirical risk minimizer (for the weighted quadratic loss). The rest of this proof follows the standard procedure for proving generalization error using Rademacher complexity \cite{Latala05a,gnecco2008approximation,bartlett2003rademacher}.

We will show that, with probability at least $1-\delta$:
\begin{align}\label{e:ULLN}
\forall \X\in \calX: \qquad \abs{ R_{l}(\X) - R_{l}(\X;\bOmega) } \le \epsilon/2,
\end{align}
where
\begin{align*}
\epsilon:= 2 (2-\rho)  \left(
2 \min\left( t C ( 2 \sqrt{d_{\max}^*} + \sqrt[4]{s}),\ r \log_2^\alpha (d_{\max}+1)  \right) + \sqrt{\frac{s \log 2/\delta}{2}}
\right).
\end{align*}
Denote $\bar{\X}=\min_{\X}\ R_{l}(\X)$. Notice \eqref{e:ULLN} implies that, with probability at least $1-\delta$, the following inequalities hold simultaneously:
\begin{align*}
R_{l}(\hat{\X}) & \le R_{l}(\hat{\X};\bOmega) + \epsilon/2, \\
R_{l}(\bar{\X};\bOmega) & \le R_{l}(\bar{\X}) + \epsilon/2.
\end{align*}
Realizing that $R_{l}(\hat{\X};\bOmega) \le R_{l}(\bar{\X};\bOmega)$ as $\hat{\X}$ is the empirical risk minimizer, the above two inequality guarantee that, with probability at least $1-\delta$, we have
\begin{align*}
R_{l}(\hat{\X}) & \le \min_{\X}\ R_{l}(\X) + \epsilon,
\end{align*}
which gives the desired generalization error bound.

We now prove \eqref{e:ULLN} and it suffice to show one side of the inequality, i.e., with probability at least $1-\delta/2$:
\begin{align}
\sup_{\X\in\calX} \left\{ R_{l}(\X) - R_{l}(\X;\bOmega) \right\} \le \epsilon/2.
\end{align}
as the proof of the other side follows the same procedure and we obtain \eqref{e:ULLN} using union bound.
First observe that $R_l (\Xij)$ can be either $\left( 1 - \frac{\rho}{2} \right) \left(\Xij-1 \right)^2$ (if $\Oij=1$) or $\frac{\rho}{2} \Xij^2$ (if $\Oij=0$). When changing the random variable $\Oij$, the change in $\sup_{\X\in\calX}\ \left\{ R_{l}(\Xij) - R_{l}(\Xij;\Oij) \right\}$ is at most
\begin{align*}
\abs{ \left( 1 - \frac{\rho}{2} \right) \left(\Xij-1 \right)^2 - \frac{\rho}{2} \Xij^2  } \le L \abs{1-0} = L,
\end{align*}
where $L:=\sup_{0\le \Xij\le 1}\ \abs{\frac{\partial \left( 1 - \frac{\rho}{2} \right) \left(\Xij-1 \right)^2 - \frac{\rho}{2} \Xij^2 }{\partial \Xij}} = 2-\rho$ is the upper bound of Lipschitz constant of the change in loss. Applying McDiarmid's Theorem \cite{mcdiarmid1989method,rio2013mcdiarmid} to $R_{l}(\X) - R_{l}(\X;\bOmega) =\sum_{i,j} R_{l}(\Xij) - R_{l}(\Xij;\Oij)$, we have with probability at least $1-\delta/2$,
\begin{align*}
\sup_{\X\in\calX} \left\{ R_{l}(\X) - R_{l}(\X;\bOmega) \right\} \le
\ExpO{ \sup_{\X\in\calX} \left\{ R_{l}(\X) - R_{l}(\X;\bOmega) \right\} } + (2-\rho) \sqrt{\frac{s \log 2/\delta}{2}},
\end{align*}
where we have used the fact that $\Oij$ can only change its value if $\Eij=1$ and there are $s$ such entries in $\E$.

We now finally upper bound the expectation
\begin{align}
 &\ExpO{ \sup_{\X\in\calX} \left\{ R_{l}(\X) - R_{l}(\X;\bOmega) \right\} } \nonumber \\
= & \ExpO{ \sup_{\X\in\calX} \left\{ \bbE_{\bOmega'}\left[ R_{l}(\X;\bOmega') \right] - R_{l}(\X;\bOmega) \right\} } \nonumber \\
\le &  \bbE_{\bOmega,\bOmega'} \left[  \sup_{\X\in\calX} \left\{ R_{l}(\X;\bOmega')  - R_{l}(\X;\bOmega) \right\} \right] \nonumber \\
= & \bbE_{\bOmega,\bOmega'} \left[  \sup_{\X\in\calX} \left\{ \sum_{i,j:\Eij=1} R_{l}(\Xij;\bOmega'_{ij}) - R_{l}(\Xij;\Oij) \right\} \right] \label{e:Rademacher-1} \\
= & \bbE_{\bOmega,\bOmega',\sigma} \left[  \sup_{\X\in\calX} \left\{ \sum_{i,j:\Eij=1} \sigma_{ij} \left( R_{l}(\Xij;\bOmega'_{ij}) - R_{l}(\Xij;\Oij) \right) \right\} \right] \nonumber \\
= & \bbE_{\bOmega,\bOmega',\sigma} \left[  \sup_{\X\in\calX} \left\{ \sum_{i,j:\Eij=1} \sigma_{ij} R_{l}(\Xij;\bOmega'_{ij}) \right\} + \sup_{\X\in\calX} \left\{ \sum_{i,j:\Eij=1} - \sigma_{ij} R_{l}(\Xij;\Oij) \right\} \right] \nonumber \\
= & 2 \bbE_{\bOmega,\sigma} \left[  \sup_{\X\in\calX} \left\{ \sum_{i,j:\Eij=1} \sigma_{ij} R_{l}(\Xij;\bOmega_{ij}) \right\} \right] \label{e:Rademacher-2}
\end{align}
where $\sigma_{ij}$ are random variables that take value in $\{-1,1\}$ with equal probability, and we have used the fact that $\Oij=\Oij'=0$ if $\Eij=0$ in \eqref{e:Rademacher-1}.

Note when $\Eij=1$, $R_{l}(\Xij;\bOmega_{ij}) = \left( 1 - \frac{\rho}{2} \right) \left(\Xij-1 \right)^2$ with probability $\rho$, and  $R_{l}(\Xij;\bOmega_{ij}) = \frac{\rho}{2} \Xij^2$ with probability $1-\rho$; in both cases, $R_{l}(\Xij;\bOmega_{ij})$ has a Lipschitz constant at most $2-\rho$ for $0\le \Xij\le 1$. Continuing from \eqref{e:Rademacher-2}, we have
\begin{align}
 \ExpO{ \sup_{\X\in\calX} \left\{ R_{l}(\X) - R_{l}(\X;\bOmega) \right\} }  \le  2 (2-\rho) \bbE_{\bOmega,\sigma} \left[  \sup_{\X\in\calX} \left\{ \sum_{i,j:\Eij=1} \sigma_{ij} \Xij \right\} \right].
\end{align}
We can now use the two regularizers to bound the Rademacher complexity. First, applying the duality of the (matrix) 2-norm and trace-norm, we have
\begin{align}
 \bbE_{\bOmega,\sigma} \left[  \sup_{\X\in\calX} \left\{ \sum_{i,j:\Eij=1} \sigma_{ij} \Xij \right\} \right] \le \bbE_{\sigma} \left[ \sup_{\norm{\X}_*\le t} \norm{\PP(\sigma,\E)}_2 \norm{\X}_* \right]
\le t \bbE_{\sigma} \left[  \norm{\PP(\sigma,\E)}_2 \right]
\label{e:Rademacher-3}
\end{align}
where $\PP(\sigma,\E)\in \bbR^{p\times p}$ is a matrix with $\PP_{ij}=\sigma_{ij}$ if $\Eij=1$ and $0$ otherwise. Applying the main theorem of \cite{Latala05a} to $\PP$, which is an independent zero mean random matrix, we have
\begin{align*}
\bbE\left[\norm{\PP}_2\right] & \le C \left( \max_i \sqrt{\sum_j\bbE\left[\PP_{ij}^2 \right]}
+ \max_j \sqrt{\sum_i\bbE\left[\PP_{ij}^2 \right]} + \sqrt[4]{\sum_{ij}\bbE\left[\PP_{ij}^4 \right]} \right)\\
& = C \left( 2 \sqrt{d_{\max}^*} + \sqrt[4]{s} \right)
\end{align*}
where $C$ is a universal constant. Continuing from \eqref{e:Rademacher-3}, we have
\begin{align}
\bbE_{\bOmega,\sigma} \left[  \sup_{\X\in\calX} \left\{ \sum_{i,j:\Eij=1} \sigma_{ij} \Xij \right\} \right] \le t C \left( 2 \sqrt{d_{\max}^*} + \sqrt[4]{s} \right). \label{e:Rademacher-4}
\end{align}

On the other hand, we could use the duality of the (vector) $\ell_1$-norm and  $\ell_\infty$-norm and obtain
\begin{align}
 \bbE_{\bOmega,\sigma} \left[  \sup_{\X\in\calX} \left\{ \sum_{i,j:\E_{i,[j]}=1} \sigma_{i,[j]} \X_{i,[j]} \right\} \right] & =  \bbE_{\bOmega,\sigma} \left[  \sup_{\X\in\calX} \left\{ \sum_{i,j:\E_{i,[j]}=1} \left( \frac{H_{j}^\alpha}{H_{d_i}^\alpha} \X_{i,[j]} \right)   \left( \frac{H_{d_i}^\alpha}{H_{j}^\alpha} \sigma_{i,[j]} \right) \right\} \right] \nonumber \\
& \le  r \frac{H_{d_{\max}}^\alpha}{H_{1}^\alpha} = r \log_2^\alpha (d_{\max}+1), \label{e:Rademacher-5}
\end{align}
where we have used the fact that $\frac{H_{j}^\alpha}{H_{d_i}^\alpha} \X_{i,[j]}\ge 0, \forall i,j$ and so the $\ell_1$ norm of $\left[ \frac{H_{j}^\alpha}{H_{d_i}^\alpha} \X_{i,[j]}\right]$ (viewing it as a vector of $p(p-1)$ coordinates) reduces to the sum. Since both \eqref{e:Rademacher-4} and \eqref{e:Rademacher-5} hold, we can use the smaller of them to bound the Rademacher complexity.
\end{proof}

\twocolumn

\bibliography{Network_Recovery}
\bibliographystyle{apalike}

\balancecolumns
\end{document}

%% file: MACPdef.tex
\newcommand{\A}{\ensuremath{\mathbf{A}}}
\newcommand{\B}{\ensuremath{\mathbf{B}}}

\newcommand{\E}{\ensuremath{\mathbf{E}}}

\newcommand{\M}{\ensuremath{\mathbf{M}}}

\newcommand{\PP}{\ensuremath{\mathbf{P}}}

\renewcommand{\SS}{\ensuremath{\mathbf{S}}}

\newcommand{\U}{\ensuremath{\mathbf{U}}}

\newcommand{\X}{\ensuremath{\mathbf{X}}}
\newcommand{\Y}{\ensuremath{\mathbf{Y}}}

\newcommand{\x}{\ensuremath{\mathbf{x}}}

\newcommand{\bbE}{\ensuremath{\mathbb{E}}}

\newcommand{\bbR}{\ensuremath{\mathbb{R}}}

\newcommand{\calX}{\ensuremath{\mathcal{X}}}

\newcommand{\abs}[1]{\left\lvert#1\right\rvert}
\newcommand{\norm}[1]{\left\lVert#1\right\rVert}

{%
\begin{list}{#1}{
\vspace{-\topsep}
\vspace{-\partopsep}
\setlength{\itemindent}{0cm}
\setlength{\rightmargin}{0cm}
\setlength{\listparindent}{0cm}
\settowidth{\labelwidth}{#1}
\setlength{\leftmargin}{\labelwidth}
\addtolength{\leftmargin}{\labelsep}
\setlength{\itemsep}{0cm}
}%
}%
{%
\end{list}
\vspace{-\topsep}
\vspace{-\partopsep}
}

{\begin{enumerate}%
\renewcommand{\theenumi}{\roman{enumi}}%
}%
{\end{enumerate}}



%% file: MACPdef-ams.tex
\newcommand{\traceop}{\operatorname{tr}}
\newcommand{\trace}[1]{\ensuremath{\traceop\left(#1\right)}}
